%% file: main.tex
\documentclass{article}

\usepackage{PRIMEarxiv}

\usepackage[utf8]{inputenc} 
\usepackage[T1]{fontenc}    
\usepackage{hyperref}       
\usepackage{url}            
\usepackage{booktabs}       
\usepackage{amsfonts,amsthm,amsmath,bbm}       
\usepackage{dsfont}
\usepackage[linesnumbered]{algorithm2e}
\usepackage{nicefrac}       
\usepackage{microtype}      
\usepackage{lipsum}
\usepackage{fancyhdr}       
\usepackage{graphicx}       
\usepackage{xparse}         
\graphicspath{{media/}}     

\title{Distribution-Specific Auditing For Subgroup Fairness
}

\author{
  Daniel Hsu \\
  Columbia University \\
  New York, NY, USA\\
  \texttt{djhsu@cs.columbia.edu} \\
   \And
  Jizhou Huang \\
  Washington University in St. Louis \\
  St. Louis, MO, USA\\
  \texttt{huang.jizhou@wustl.edu} \\
  \And
  Brendan Juba \\
  Washington University in St. Louis \\
  St. Louis, MO, USA\\
  \texttt{bjuba@wustl.edu} \\
}
\input{commands}

\RestyleAlgo{ruled}
\IncMargin{1em}

\begin{document}
\maketitle

\begin{abstract}
    We study the problem of auditing classifiers for statistical subgroup fairness. Kearns et al.~\cite{kearns2018preventing} showed that the problem of auditing combinatorial subgroups fairness is as hard as agnostic learning. Essentially all work on remedying statistical measures of discrimination against subgroups assumes access to an oracle for this problem, despite the fact that no efficient algorithms are known for it. If we assume the data distribution is Gaussian, or even merely log-concave, then a recent line of work has discovered efficient agnostic learning algorithms for halfspaces. Unfortunately, the reduction of Kearns et al.\ was formulated in terms of weak, ``distribution-free'' learning, and thus did not establish a connection for families such as log-concave distributions.
    In this work, we give positive and negative results on auditing for Gaussian distributions: On the positive side, we present an alternative approach to leverage these advances in agnostic learning and thereby obtain the first polynomial-time approximation scheme (PTAS) for auditing nontrivial combinatorial subgroup fairness: we show how to audit statistical notions of fairness over homogeneous halfspace subgroups when the features are Gaussian. On the negative side, we find that under cryptographic assumptions, no polynomial-time algorithm can guarantee any nontrivial auditing, even under Gaussian feature distributions, for general halfspace subgroups.
\end{abstract}

\section{INTRODUCTION}
    The deployment of decision rules obtained using machine learning has raised the risk that the rules may exhibit biases against historically marginalized communities. In particular, Kearns et al.~\cite{kearns2018preventing} raised the concern that these decision rules may be biased against sub-groups characterized by a combination of ``protected'' attributes. Since there are an exponential number of such subgroups, even detecting such statistical patterns of discrimination is a nontrivial computational problem; indeed, Kearns et al.~\cite{kearns2018preventing} showed that the problem of finding disadvantaged subgroups is equivalent to the problem of agnostic learning, which is believed to be intractable in general for all but the simplest classes of sets. Essentially all work \cite{kearns2018preventing,kim2018fairness,hebert2018multicalibration} on remedying statistical measures of discrimination against subgroups assumes access to an oracle for this problem, despite the fact that no efficient algorithms are known for it. In this work we are proposing a solution for a variant of the fairness auditing problem with provable guarantees of efficiency and correctness, as well as some strong limitations on the extent to which these solutions can be extended to richer families of subgroups.
    \subsection{Background and Motivation}
    Fairness learning has received massive attention in recent years. It turns out learning a fair classifier, in most cases, is equivalent to auditing \cite{kearns2018preventing,kim2018fairness,hebert2018multicalibration}. In particular, if auditing is possible, learning a fair classifier is easy. There are many successful examples of fairness learning with auditing over a relatively small number of predetermined subgroups \cite{agarwal2018reductions, wang2023scalable}. However, a small number of predetermined subgroups, in many cases, is not enough to cover all the natural subgroups.
    
    \begin{example}
        In the court case ``DeGraffenreid v General Motors" \cite{crenshaw2013demarginalizing}, five Black women brought suit against General Motors for its discrimination against the group of Black women. Although no sex discrimination was revealed, the evidence showed that Black women hired after 1970 were discriminated against by the company's seniority system. Such discrimination can be better demonstrated by an example shown in table~\ref{tbl:an-example-of-discrimination-against-subgroups}. In particular, the hiring rate of a company could seemingly be fair in terms of gender of race alone but clearly discriminates against the subgroups of white men and black women. As a result, the court rejected the plaintiffs' attempt to bring a suit not on behalf of Blacks or women, but specifically on behalf of Black women. In the ruling, in favor of the defendant, the judge was specifically concerned about the proliferation of protected classes. 
        
        \begin{table}[ht]
            \caption{an example of discrimination against subgroups}\label{tbl:an-example-of-discrimination-against-subgroups}
            \begin{center}
                \begin{tabular}{|c|c|c|c|}
                  \hline
                   & men&women& total\\
                  \hline
                  black & 50 & 0&50\\
                  \hline
                  white & 0 & 50&50\\
                  \hline
                  total&50&50&100\\
                  \hline
                \end{tabular}
            \end{center}      
        \end{table}
        
    \end{example}
    
    More generally, a classifier may appear to be fair on each individual attribute, e.g., gender, race, age, incomes, etc., and yet perform unfairly on subgroups defined on multiple attributes, i.e., the conjunction of such attributes. In the case of \emph{DeGraffenreid v General Motors}, it is the conjunction of race and gender being discriminated against. The possible number of the conjunctions grows exponentially as the number of the ``protected'' attributes increases.
    
    Thereafter, \cite{kearns2018preventing} proposed more general notions of statistical fairness that require auditing over subgroups defined on simple combinations of data features. Specifically, such combinations of features can be any simple representations, such as conjunctions and halfspaces, which, however, can generate exponentially many subgroups. They also showed that the problem of auditing subgroups defined by such simple representation is as hard as ``weak agnostic learning'' in the standard ``distribution-free'' setting \cite{haussler1992decision,kearns1994toward}. While the problem of distribution-free weak agnostic learning is widely believed to be computationally intractable \cite{kearns1994toward, feldman2009agnostic}, its hardness does not necessarily hold for specific distribution families. Thus, it is natural to consider auditing using distribution-specific agnostic learning approaches as agnostic learning is a much more extensively studied problem. However, it turns out there are still obstacles remaining for doing so. 
    
    \subsection{Challenges of Auditing through Agnostic Learning}
    
    The main challenge that prevents us from applying existing agnostic learning techniques to perform auditing based on the reduction by \cite{kearns2018preventing} is that it is formulated in terms of weak agnostic learning, that is, finding classifiers with error rates that are nonnegligibly better than guessing, and correspondingly weak auditing guarantees. In particular, the approximation guarantees we obtain for distribution-specific agnostic learning yield vacuous guarantees for weak learning. 
    When we have guarantees for arbitrary distributions, ``boosting'' \cite{schapire1990strength} enables us to obtain high accuracy from such weak learners. 
    Unfortunately, these techniques require re-weighting the data examples after which the distribution-specific properties may no longer hold.
    
    One might hope to dodge this issue by casting the problem of finding a harmed subgroup as a Mixed-Integer Program and using solvers that, though they lack polynomial-time guarantees, obtain adequate performance in practice. In such an approach, the failure of the solver to find a feasible solution to the optimization problem is taken as the proof that the classifier is fair. Unfortunately, these solvers owe their speed in part to a lack of soundness, both due to numerical issues \cite{cook2013hybrid} and the complexity of the heuristics used to prune the search \cite{akgun2018metamorphic,gillard2019solvercheck}, and it remains a current research challenge to obtain acceptable performance (using the various advanced techniques employed by commercial solvers) while retaining the guarantee that the solver correctly reports infeasibility \cite{bogaerts2023certified}. In any case, the works by \cite{kearns2018preventing,kearns2019empirical} and \cite{kim2019multiaccuracy} that empirically studied these approaches to obtaining fair classifiers used linear regression as a proxy for the agnostic learning or cost-sensitive classification subroutines. Unfortunately, these heuristics do not even provide in-principle guarantees.
    
    In this paper, we will show auditing general halfspace subgroups is hard even under Gaussian distribution and present an alternative auditing approach for subgroups determined by homogeneous halfspaces with provable guarantees.
    
    \subsection{Our Contribution}
    Our first contribution is a more careful analysis of the relationship between auditing and agnostic learning:
    Given a fixed positive classification rate, the harm (w.r.t.\ statistical parity) suffered by a subgroup is affinely related to the error rate of the subgroup indicator. Thus, a solution to the agnostic learning problem directly gives a harmed subgroup.
    It is worth noting that under a standard normal distribution, the subclass of halfspaces with a fixed positive classification rate is given by the halfspaces with unit normal vectors and the same threshold.
    
    \begin{remark}
        Our reduction to learning halfspaces with fixed positive classification rates can achieve arbitrarily high precision auditing and does not rely on re-weighting data examples or make any assumptions on the potentially unfair classifiers. This enables the use of the existing distribution-specific agnostic learning methods for auditing. 
    \end{remark}
    
    Based on the reduction and a inspiration from Diakonikolas et al.~\cite{diakonikolas2023near}, our second major contribution is a lower bound on the unfairness detectable when auditing for halfspace subgroups under Gaussian distributions by reducing the problem of continuous Learning With Errors (cLWE) to auditing. Our hardness results include both multiplicative and additive forms. More interestingly, we can further show that even ``nonconstructive auditing'' is hard, where we do not need to exhibit a discriminated subgroup for a failed audit.
    
    For our algorithmic results, we will present a general auditing framework given an oracle for (distribution-specific) agnostic learning. Also, we give a randomized PTAS auditing algorithm for subgroups determined by homogeneous halfspaces under Gaussian data by applying the method from Diakonikolas et al.~\cite{diakonikolas2021agnostic}. 
    
    \begin{remark}
        We stress that a PTAS for auditing subgroups defined by homogeneous halfspaces for Gaussian distributions is, in fact, the best guarantee we know so far, hence, not trivial.
    \end{remark}
    
    At first blush, the reliance on a (prima facie unverifiable) distributional assumption for the analysis of our auditing algorithm may seem to be at odds with our desire to certify the fairness of a classifier. Nevertheless, a line of recent works by Rubinfeld and Vasilyan~\cite{rubinfeld2023testing} and Gollakota et al.~\cite{gollakota2023moment} have shown that the properties of the data that are crucial to these algorithms for distribution-specific learning of halfspaces \emph{can be verified}. Thus, these methods give a way of certifying fairness for families of nice distributions: so long as the data passes these tests and the audit reveals no subgroup that is significantly harmed, we may \emph{guarantee that the classifier is fair}.
    
    This paper will be organized as follow. Some necessary background for our arguments are given in Section 2. We will present the main reduction from auditing to agnostic learning in Section 3. Then, we will show the hardness results in Section 4. Section 5 will present out auditing framework as well as the distribution-specific PTAS algorithm. At last, we will discuss the limitation of our approach and our future work.
    
    \subsection{Related Work}
    Many authors have considered the problem of ensuring fairness in classification, and Barocas et al.~\cite{barocas-hardt-narayanan} give a good overview of the broader area. In particular, there are alternatives to the statistical, group-fairness notions we are considering, for example individual-level fairness as proposed by Dwork et al.~\cite{dwork2012fairness}, or based on causal modeling, such as the ``counterfactual'' fairness notion proposed by Kusner et al.~\cite{kusner2017}. We cannot do justice to the breadth of literature and philosophical issues here, and we strongly encourage the interested reader to consult Barocas et al. The group-fairness notions we consider have their roots in the game-theory-based approach of Kearns et al.~\cite{kearns2018preventing} for learning representations with subgroup fairness by assuming there exists an efficient oracle for auditing. A follow-up study \cite{kearns2019empirical} evaluated their algorithm on real-world datasets. H\'{e}bert-Johnson et al.~\cite{hebert2018multicalibration} showed a method of obtaining ``multi-accurate'' representations by assuming the existence of an efficient auditing oracle. Further, Kim et al.~\cite{kim2018fairness} proposed a variant of statistical fairness called ``multi-fairness,'' which allows them to efficiently learn a multi-fair classifier with querying ``relative fairness'' of data pairs. As we discussed previously, the auditing oracles in these works were provided by using linear regression as a heuristic for the optimal halfspace, which does not provide guarantees. They also did not consider auditing for specific families of distributions. On the other hand, the works on agnostic learning for specific families of distributions, e.g., \cite{kalai2008agnostically,diakonikolas2020non,diakonikolas2021agnostic,diakonikolas2022learning,frei2021agnostic} do not consider how their techniques may be applied to the subgroup fairness auditing problem.

\section{PRELIMINARIES}
    We use lowercase bold font characters to represent real vectors and subscripts to index the coordinates of each vector, e.g., $\xscript<i>$ represents the $i$-th coordinate of vector $\xscript$. We denote the $l_p$-norm by $\norm{\xscript}_p = \sbr{\sum_i\xscript<i>(p)}^{1/p}$, and $\xscript* = \xscript/\norm{\xscript}_2$. We model each individual as a vector of protected attributes, i.e., $\xscript\in\mathcal{X}$.
    
    Further, the probability of an event under a distribution $\D$ is denoted by $\prob<\xscript\sim\D>{\cdot}$. $\gaussian(0, \identity)$ denotes a standard normal distribution, where $\identity$ represents the identity matrix. For simplicity of notation, we may use $\gaussian, \gaussian_\sigma$ instead of $\gaussian(0, \identity), \gaussian(0, \sigma^2\identity)$ or even drop $\D$ and $\gaussian$ from the subscript when it is clear from the context. 
    \begin{fact}[Rotational Invariance]\label{fac:gaussian-rotational-invariance-property}
        For any real vector $\uscript$, if $\xscript\sim\gaussian(0, \identity)$, then $\uscript*(\transpose) \xscript\sim \gaussian(0, 1)$.
    \end{fact}
    
    To understand the problem of fairness auditing, it is necessary to define fairness or unfairness precisely. In this work, we focus on the notion of Statistical Parity Subgroup Fairness (SPSF). Formally, we have the following definition.
    
    \begin{definition}[Statistical Parity Subgroup Fairness]\label{def:SP-subgroup-unfairness-definition}
        Fix any binary classifier $c\in\C$ such that $c:\R^d\rightarrow\lbr{-1, +1}$, data distribution $\D$, collection of subgroups $\GI$, and parameter $\gamma\in[0,1]$. Define
        \begin{equation}
            d_\D(c,g) = \prob<\xscript\sim\D>{c(\xscript)=1} - \prob<\xscript\sim\D>{c(\xscript)=1\cond \xscript\in g}\label{eq:definition-of-deviation}
        \end{equation}
        We say that g does not satisfy $\gamma$-statistical parity fairness (or is $\gamma$-unfair) with respect to $\D$ and $\GI$, if $\exists g\in \GI$ such that
        \begin{equation}
          \prob<\xscript\sim\D>{\xscript\in g}\abs{d_\D(c,g)}\geq \gamma\label{eq:SP-subgroup-unfairness-definition}
        \end{equation} 
    \end{definition}
    
    Equation \eqref{eq:definition-of-deviation} is a straightforward way to quantify how much the positive classification rate within a subgroup deviates from that of the overall population. The weighting by the size of the group (i.e., $\prob<\xscript\sim\D>{\xscript\in g}$) is a concession to address the statistical issues that arise with estimating $d$ on small groups: we cannot escape that our empirical estimates are less accurate as the size shrinks.
    The goal of fairness auditing is to develop an ``auditing algorithm" to efficiently find such a certificate $g\in\GI$ for any $c\in\C$ with sample access to $\D$, formalized as follows.
    
    \begin{definition}[Constructive Auditing \cite{kearns2018preventing}]\label{def:auditing-algorithm-for-sp-subgroup-fairness}
        Fix a collection of group indicators $\GI$ over the protected features, and any $\delta, \gamma,\gamma'\in(0,1)$ such that $\gamma'\leq \gamma$. A constructive $(\gamma, \gamma')$-auditing algorithm for $\GI$ with respect to distribution $\D$ is an algorithm $\algo$ such that for any classifier $h$, when given access the joint distribution $(\D, h(\D))$, $\algo$ runs in time $\mathrm{poly}(1/\gamma', \log(1/\delta))$, and with probability $1-\delta$, outputs a $\gamma'$-unfair certificate for $h$ whenever $h$ is $\gamma$-unfair with respect to $\D$ and $\GI$. If $h$ is $\gamma'$-fair, $\algo$ will output ``fair''.
    \end{definition}

    Moreover, we will consider a more general type of auditing task, called ``non-constructive auditing'', where the algorithms are only required to tell if a discriminated subgroup exists.

    \begin{definition}[Non-constructive Auditing]\label{def:non-constructive-auditing}
        Under the same setting as Definition \ref{def:auditing-algorithm-for-sp-subgroup-fairness}, a non-constructive $(\gamma, \gamma')$-auditing algorithm for $\GI$ with respect to distribution $\D$ is an algorithm $\algo$ such that for any classifier $h$, when given access the joint distribution $(\D, h(\D))$, $\algo$ runs in time $\mathrm{poly}(1/\gamma', \log(1/\delta))$, and with probability $1-\delta$, claims $h$ is $\gamma'$-unfair whenever $h$ is $\gamma$-unfair with respect to $\D$ and $\GI$. If $h$ is $\gamma'$-fair, $\algo$ will output ``fair''.
    \end{definition}
    
    Since our reduction involves the subclass of halfspace subgroups of a fixed size, we give the formal definition of it as follows.
    \begin{definition}[Fixed-size Halfspaces]\label{def:halfspaces-with-proabilistic-lock}
        We use $\HS^d$ to represent the collection of all the halfspaces in $\R^d$. Then, for any arbitrary distribution $\D$ over $\R^d$, we define the collection of all halfspaces with the same (relative) density $\mu$ as
        \begin{equation}
            \HS_{\mu}^\D:=\lbr{h\in\HS^d\cond \Pr_{\xscript\in \D}\lbr{h(\xscript) = 1} = \mu }
        \end{equation}
    \end{definition}
    For conciseness, we may abbreviate $\Pr\lbr{f(\xscript) = 1}$, $\Pr\lbr{f(\xscript) = -1}$ to simply $\Pr\lbr{f}$, $\Pr\lbr{\neg f}$ for any binary output functions $f:\mathcal{X}\rightarrow\lbr{-1, +1}$ when it is necessary for the rest of the paper.
    
    To state the hardness results, we denote $\Sphere^{d-1}:=\lbr{\xscript\in \R^d \cond \norm{\xscript}_2= 1}$, $\Z_q:=\lbr{0, 1, \ldots, q-1}$, $\R_q:=[0, q)$, and $\mod_q:\R^d\rightarrow \R_q$ by the unique translation by $q\mathbb{Z}^d$ for $q\in\N$. Then we formally define the problem of ``learning with errors'' (LWE) \cite{regev2009lattices}, following \cite{diakonikolas2023near}:
    \begin{definition}[Learning With Errors]
        For $m,d\in\N$, $q\in\R_+$, let $\D_{sp},\D_{st},\D_{ns}$ be distribution on $\R^d, \R^d, \R$ respectively. In the LWE$(m, \D_{sp},\D_{st},\D_{ns}, \mod_q)$ problem, with $m$ independent samples $(\xscript,y)$, we want to distinguish between the following two cases:
        \begin{itemize}
            \item \textbf{Alternative hypothesis}: $(\xscript, y)$ is generated as $y= \mod_q(\sscript(\transpose)\xscript + z)$, $\xscript\in\D_{sp}, \sscript\in\D_{st}, z\in\D_{ns}$.
            \item \textbf{Null hypothesis}: $\xscript\in\D_{sp}$,  $y$ is sampled uniformly at random on the support of its marginal distribution in alternative hypothesis, independent of $\xscript$.
        \end{itemize}
    \end{definition}
    This problem is widely believed to be computationally hard, formalized as follows.
    \begin{assumption}[Sub-exponential LWE Assumption]\label{asp:sub-exponential-assumption-of-lwe}
        For $q,\kappa\in\N, \alpha\in(0, 1)$ and $C > 0$ being a sufficiently large constant, the problem LWE$(2^{O(n^\alpha)}, \Z_q^d, \Z_q^d, \gaussian_\sigma, \mod_q)$ with $q\leq d^{\kappa}$ and $\sigma = C\sqrt{d}$ cannot be solved in $2^{O(d^\alpha)}$ time with $2^{O(-d^\alpha)}$ advantage.
    \end{assumption}

\section{FROM AUDITING TO AGNOSTIC LEARNING}
    In this section, we describe our reduction from auditing to agnostic learning. In addition, we give a lower bound for fairness auditing under Gaussian distributions.
    
    \subsection{Reduction from Auditing to Halfspace Learning}
    
    We are considering the auditing problem w.r.t.\ SPSF as in definition \ref{def:SP-subgroup-unfairness-definition}, which naturally rules out the statistically small subgroups. Indeed, if the probability of accessing the data of certain sub-population is exponentially small, it is computationally hard to even estimate their deviation. Therefore, it makes sense to just consider the collection of subgroups $\GI$ that are statistically large enough, e.g., $\Pr\lbr{\xscript\in g}=\Theta(1)$ for $\xscript\in\R^d$.
    
    Based on the observation, the following optimization program, $\Pg_{a,b}^\D(\HS^d)$, can capture the most unfair subgroup which is also statistically significant enough. That is
    \begin{align}
        \max_{g\in\GI}\quad &\Pr_{\xscript\in\D}\lbr{\xscript\in g}\abs{d_\D(c, g)}\notag\\
        s.t.\quad &a \leq\Pr_{\xscript\in\D}\lbr{\xscript\in g}\leq b\label{eq:original-optimization-program}
    \end{align}
    for some constants $0 < a\leq b < 1$. 
    
    Furthermore, if we only consider the subgroups represented by halfspaces, there exists a simple reduction from $\Pg_{a,b}^\D(\HS^d)$ to agnostic learning that, in particular, preserves the properties of the data distribution. We show our reduction as the following theorem.
    
    \begin{theorem}[Main Reduction]\label{thm:main-reduction}
        Given any binary classifier $c:\R^d\rightarrow\lbr{-1, +1}$, and a data distribution $\D$ over $\R^d$ whose 1-dimensional marginals have continuous cumulative distribution functions, if there exists an efficient algorithm for learning $\HS_\mu^\D$ in the agnostic model on distribution $\D$, then there is an efficient auditing algorithm for $c$ on subgroups represented by $\HS^d$ over distribution $\D$.
    \end{theorem}
    We delay the proof of the above theorem to the end of this section, and show two fundamental hurdles we need to overcome in order to prove Theorem \ref{thm:main-reduction}.
    
    \begin{remark}
        While learning from a representation class like $\HS_\mu^\D$ may seems to be hard at a first glance, there are actually examples \cite{diakonikolas2022learning} of learning $\HS_\mu^\D$ in an agnostic setting under Gaussian data. 
    \end{remark}
    
    Instead of starting from the optimization problem \eqref{eq:original-optimization-program}, it turns out that solving a sequence of simpler optimization problems suffices to certify the $\gamma$-unfairness as stated in Definition~\ref{def:SP-subgroup-unfairness-definition}. We prove the equivalence as the following proposition.
    \begin{proposition}\label{prop:approximating-auditing-with-simple-optimization}
        Consider any
        binary classifier $c:\R^d\rightarrow\lbr{-1, +1}$,
        any data distribution $\D$ over $\R^d$ whose 1-dimensional marginals have continuous cumulative distribution functions,
        and any $0 < a \leq b < 1$. 
        For each pair of non-negative integers $k < n$, 
        let $\Pg_{a,b}^\D(k, n)$ denote the optimization program
        \begin{align*}
            \max_{h\in\HS^d}\quad &\Pr_{\xscript\in\D}\lbr{h(\xscript) = 1}\abs{d_\D(c, h)}\\
            s.t.\quad &\Pr_{\xscript\in\D}\lbr{h(\xscript) = 1}=a + \frac{k(b-a)}{n}
            .
        \end{align*}
        Let $h^*$ be a global optimizer of $\Pg_{a,b}^\D(\HS^d)$, as defined in \eqref{eq:original-optimization-program}, and let $\gamma^* = \Pr\lbr{ h^*}\abs{d_\D(c, h^*)}$.
        For each $k=0, \ldots, n$, let $h_k^*$ be a global optimizer of $\Pg_{a,b}^\D(k, n)$.
        Then
        \begin{equation*}
            \max_k \Pr\lbr{ h_k^*}\abs{d_\D(c, h_k^*)} \geq \gamma^* - \frac{2(b-a)}{n}
            .
        \end{equation*}%
    \end{proposition}
    \begin{proof}
        For conciseness of the proof, we define
        \begin{equation*}
             \alpha(k):= a + \frac{k(b-a)}{n}
        \end{equation*}
        Since $a\leq \Pr\lbr{h^*(\xscript) = 1}\leq b$ by definition, there must exists a $k\in\lbr{0, \ldots, n-1}$ such that
        \begin{equation*}
            \alpha(k) < \Pr\lbr{h^*(\xscript) = 1} < \alpha(k + 1)
        \end{equation*}
    
        Then, since we assumed that $\D$ has a continuous CDF w.r.t.\ the normal of $h^*$, we can construct another halfspace $h'$ by either increasing or decreasing the threshold of $h^*$ until $\Pr\lbr{\xscript\in h'}$ hits either $\alpha(k)$ or $\alpha(k + 1)$. We thus obtain 
        \begin{align}
            \Pr\lbr{h'(\xscript) \neq h^*(\xscript)} =& |\Pr\lbr{h^*} - \Pr\lbr{h'}|\notag\\
            \leq& \alpha(k+ 1) - \alpha(k)\notag\\
            =& \frac{(b-a)}{n}\label{eq:upper-bound-on-the-difference-of-positive-rate-between-parallel-halfspaces}
        \end{align}
        Let $\dom:=\lbr{\xscript\cond h'(\xscript) \neq h^*(\xscript)}$. Then, by the triangle inequality 
        and the fact that $\Pr\lbr{c(\xscript) = 1}\leq 1$, we have
        \begin{align}
            \abs{\Pr\lbr{h^*}d_\D(c, h^*)} - \abs{\Pr\lbr{h'}d_\D(c, h')}\leq& |\Pr\lbr{h^*} - \Pr\lbr{h'}| + |\Pr\lbr{h'\cap c} - \Pr\lbr{h^*\cap c}|\notag\\
            \cleq&\frac{(b-a)}{n} + |\Pr\lbr{h'\cap c\cap\dom}- \Pr\lbr{h^*\cap c\cap\dom}|\notag\\
            \leq& \frac{(b-a)}{n} + |\Pr\lbr{\xscript\in\dom}|\notag\\
            \leq& \frac{2(b-a)}{n} \label{eq:unfairness-difference-between-parallel-halfspaces}
        \end{align}
        where the second inequality is obtained by expanding $\Pr\lbr{h\cap c}$ on the event $\xscript\in\dom$ using the law of total probability and exploiting the fact that $h'$ always agrees with $h^*$ on the complement of $\dom$, i.e., $\Pr\lbr{h'\cap c\cap \dom^c} = \Pr\lbr{h^*\cap c\cap \dom^c}$; the third inequality holds because at most one of $h^*(\xscript) = 1$ and $h'(\xscript) = 1$ holds for any $\xscript\in\dom$ by definition; and the last inequality is due to equation \eqref{eq:upper-bound-on-the-difference-of-positive-rate-between-parallel-halfspaces}. 
        
        Finally, due to the optimality of $h_k^*$, we have
        \begin{align*}
            \Pr\lbr{h_k^*}\abs{d_\D(c, h_k^*)}\geq&\Pr\lbr{ h'}\abs{d_\D(c, h')} - \gamma^* + \gamma^*\\
            \geq&\gamma^* - \frac{2(b-a)}{n}
        \end{align*}
        by inequality \eqref{eq:unfairness-difference-between-parallel-halfspaces} with $\Pr\lbr{h^*(\xscript) = 1}\abs{d_\D(c, h^*)} = \gamma^*$.
    \end{proof}
    The reason why this proposition is so crucial is that it allows us to solve a simpler optimization problem without compromising the guarantee. Being able to fix $\Pr\lbr{h(\xscript) = 1}$ as a constant will significantly simplify the overall optimization as it reduces the degree of the optimization objective. In fact, it is because we can optimize $\Pr\lbr{h(\xscript) = 1}\abs{d_\D(c, h)}$ over $\HS_\mu^\D$ instead of $\HS^d$ that we can conduct the reduction from auditing to agnostic learning.
    
    The following lemma shows a direct relationship between the unfairness level and the classification error.
    \begin{lemma}\label{lma:relationship-between-auditing-and-classification-error}
        Given any binary classifier $c:\mathcal{X}\rightarrow\lbr{-1, +1}$, a data distribution $\D$ over $\mathcal{X}$ and a collection of subgroups $g\in\GI$ such that $g:\mathcal{X}\rightarrow \lbr{-1, +1}$, we have
        \begin{equation*}
            2\Pr\lbr{ g}d_\D(c, g) =\Pr\lbr{\neg c}\Pr\lbr{ \neg g} + \Pr\lbr{ c}\Pr\lbr{ g} - \Pr\lbr{c(\xscript) = g(\xscript)}
        \end{equation*}
        for $\xscript\sim\D$.
    \end{lemma}
    \begin{proof}
        By the law of total probability, we have
        \begin{equation*}
            \Pr\lbr{ c\cap  g}=\Pr\lbr{ g} - (\Pr\lbr{\neg c} - \Pr\lbr{\neg c\cap  \neg g}).
        \end{equation*}
        which along with definition \ref{def:SP-subgroup-unfairness-definition} gives
        \begin{align}
            d_\D(c, g) =&\Pr\lbr{ c}- \Pr\lbr{c \cond g}\notag\\
            =&\frac{\Pr\lbr{ c}\prob{g}- \Pr\lbr{c \cap g}}{\prob{g}}\notag\\
            =&\frac{\Pr\lbr{\neg c}\Pr\lbr{ \neg g} - \Pr\lbr{\neg c\cap  \neg g}}{\Pr\lbr{ g}}.\label{eq:negation-invariance-of-deviation-under-gaussian}
        \end{align}
        Summing up the two different forms of $d_\D(c, g)$ results to
        \begin{align}
            2d_\D(c, g) =&\frac{\Pr\lbr{\neg c}\Pr\lbr{ \neg g} - \Pr\lbr{\neg c\cap  \neg g}}{\Pr\lbr{ g}} + \frac{\Pr\lbr{ c}\prob{g}- \Pr\lbr{c \cap g}}{\prob{g}}\notag\\
            =&\frac{\Pr\lbr{\neg c}\Pr\lbr{ \neg g} + \Pr\lbr{ c}\prob{g} - (\Pr\lbr{\neg c\cap  \neg g} + \Pr\lbr{c \cap g})}{\prob{g}}\label{eq:summation-of-two-different-forms-of-deviation}
        \end{align}
        Notice that, because $c\cap g$ and $\neg c\cap \neg g$ are two disjoint events, we have
        \begin{align*}
            \Pr\lbr{c(\xscript) = g(\xscript)}=& \Pr\lbr{( c\cap g)\cup(\neg c\cap \neg g)}\\
            =&\Pr\lbr{ c\cap g} + \Pr\lbr{\neg c\cap \neg g}
        \end{align*}
        Plugging it back to equation \eqref{eq:summation-of-two-different-forms-of-deviation} produces the desired result.
    \end{proof}
    This immediately implies the duality between SPSF auditing and agnostic learning as follow.
    \begin{corollary}\label{cor:optimal-agnostic-learning-implies-optimal-auditing}
        Given any binary classifier $c:\R^d\rightarrow\lbr{-1, +1}$, a data distribution $\D$ and a collection of halfspaces $\HS_\mu^\D$ over $\R^d$, we have the following two properties
        \begin{enumerate}
            \item[(1)] $d_\D(c, h^*)\geq d_\D(c, h), \forall h\in\HS_\mu^\D$ if and only if $h^* = \argmin_{h\in\HS_\mu^\D}\Pr_{\xscript\sim\D}\lbr{c(\xscript) = h(\xscript)}$
            \item[(2)] $d_\D(c, h^*)\leq d_\D(c, h), \forall h\in\HS_\mu^\D$ if and only if $h^* = \argmax_{h\in\HS_\mu^\D}\Pr_{\xscript\sim\D}\lbr{c(\xscript) = h(\xscript)}$
        \end{enumerate}
    \end{corollary}
    \begin{proof}
        Because $\Pr\lbr{ c}$ is a constant and $\Pr\lbr{h} = \mu,\forall h\in\HS_\mu^\D$ by Definition~\ref{def:halfspaces-with-proabilistic-lock}, $d_\D(c, h)$ is simply an affine transformation of $\Pr\lbr{c(\xscript) = h(\xscript)}$ for a fixed $\mu$ by Lemma \ref{lma:relationship-between-auditing-and-classification-error}, which implies the desired results.
    \end{proof}
    
    Proposition \ref{prop:approximating-auditing-with-simple-optimization} tells us that solving $\Pg_{a,b}^\D(k, n)$ for $k=0,\ldots, n$ would give us a good enough approximation to the maximum unfairness level, of course, with a large enough $n$. Therefore, we just need to further show that solving each $\Pg_{a,b}^\D(k, n)$ is equivalent to learning $\HS_\mu^\D$ to complete the reduction.
    
    Formally, because $\Pg_{a,b}^\D(k, n)$ can be equivalently written as
    \begin{equation}
        \max_{h\in\HS_\mu^\D}\quad \Pr_{\xscript\in\D}\lbr{h(\xscript) = 1}\abs{d_\D(c, h)}
    \end{equation}
    for some $\mu = a + k(b-a)/n$, it suffices to prove the following theorem.
    \begin{lemma}\label{lma:reducing-auditing-to-agnostic-learning}
        Given any binary classifier $c:\R^d\rightarrow\lbr{-1, +1}$, a data distribution $\D$ and a collection of halfspaces $\HS_\mu^\D$ over $\R^d$ such that
        \begin{equation*}
            \opt_{\min}\leq \Pr_{\xscript\sim\D}\lbr{c(\xscript)= h(\xscript)}\leq \opt_{\max}
        \end{equation*}
        for all $h\in\HS_\mu^\D$, if $h_{\vscript<>},h_{\uscript}\in\HS_\mu^\D$ satisfy that $\Pr\lbr{c(\xscript) = h_{\vscript<>}(\xscript)}\leq\opt_{\min} + 2\epsilon$ as well as $\Pr\lbr{c(\xscript) = h_{\uscript}(\xscript)}\geq\opt_{\max} - 2\epsilon$, we have either
        \begin{equation}
            \Pr_{\xscript\sim\D}\lbr{h_{\vscript<>}(\xscript) = 1}\abs{d_\D(c, h_{\vscript<>})} \geq \gamma^* - \epsilon
        \end{equation}
        or
        \begin{equation}
            \Pr_{\xscript\sim\D}\lbr{h_{\uscript}(\xscript) = 1}\abs{d_\D(c, h_{\uscript})}\geq \gamma^* - \epsilon
        \end{equation}
        where $\gamma^* = \max_{h\in\HS_\mu^\D}\Pr_{\xscript\sim\D}\lbr{h(\xscript) = 1}\abs{d_\D(c, h)}$.
    \end{lemma}
    \begin{proof}
        By the proof of Lemma~\ref{lma:relationship-between-auditing-and-classification-error}, we have
        \begin{equation*}
            2\Pr\lbr{ h}\abs{d_\D(c, h)}=|\underbrace{\Pr\lbr{\neg c}\Pr\lbr{\neg h} - \Pr\lbr{\neg c\cap \neg h}}_{I_1}  + \underbrace{\Pr\lbr{ c}\Pr\lbr{ h} - \Pr\lbr{ c\cap  h}}_{I_2}|
        \end{equation*}
        Let $h^*\in\HS_\mu^\D$ be such that $\Pr\lbr{ h^*}\abs{d_\D(c, h^*)}=\gamma^*$. Then for $I_2$, we have
        \begin{align*}
            I_2 =& (\Pr\lbr{ c} - \Pr\lbr{ c\cond  h^*} + \Pr\lbr{ c\cond  h^*})\Pr\lbr{ h} - \Pr\lbr{ c\cap  h}\\
            =&\Pr\lbr{ h^*} d_\D(c, h^*) + \Pr\lbr{ c\cap  h^*}- \Pr\lbr{ c\cap  h}
        \end{align*}
        where the last equation is because $h^*\in \HS_\mu^\D$, then $\Pr\lbr{ h} = \Pr\lbr{ h^*} = \mu$ by Definition~\ref{def:halfspaces-with-proabilistic-lock}.
    
        Similarly, for $I_1$, we can write
        \begin{align*}
            I_1=&\Pr\lbr{\neg h^*}(\Pr\lbr{\neg c} - \Pr\lbr{\neg c\cond \neg h^*}) + \Pr\lbr{\neg c\cap \neg h^*}- \Pr\lbr{\neg c\cap \neg h}\\
            =& \Pr\lbr{ h^*} d_\D(c, h^*)+ \Pr\lbr{\neg c\cap \neg h^*}- \Pr\lbr{\neg c\cap \neg h}
        \end{align*}
        where the last equation follows because we have shown in the proof of Lemma~\ref{lma:relationship-between-auditing-and-classification-error} that $d_\D(c, h^*) = \Pr\lbr{\neg h^*}(\Pr\lbr{\neg c} - \Pr\lbr{\neg c\cond \neg h^*})/\Pr\lbr{ h^*} $.
    
        Combining $I_1$ and $I_2$ will result to
        \begin{align*}
            \Pr\lbr{ h}\abs{d_\D(c, h)}=&|\Pr\lbr{ h^*} d_\D(c, h^*) + \frac{\Pr\lbr{c(\xscript) = h^*(\xscript)}- \Pr\lbr{c(\xscript) = h(\xscript)}}{2}|\\
            \geq& \gamma^* - \frac{\abs{\Pr\lbr{c(\xscript) = h^*(\xscript)}- \Pr\lbr{c(\xscript) = h(\xscript)}}}{2}
        \end{align*}
        by triangle inequality. Further, since $h^*$ maximizes $\abs{d_\D(c, h)}$, it either maximizes or minimizes $d_\D(c, h)$. Then, by Corollary \ref{cor:optimal-agnostic-learning-implies-optimal-auditing}, we know 
        $$\Pr_{\xscript\sim\D}\lbr{c(\xscript) = h^*(\xscript)}\in \lbr{\opt_{\min}, \opt_{\max}}$$
        which implies either 
        \begin{equation*}
            \abs{\Pr\lbr{c(\xscript) = h^*(\xscript)}- \Pr\lbr{c(\xscript) = h_{\vscript<>}(\xscript)}} \leq 2\epsilon
        \end{equation*}
        or
        \begin{equation*}
            \abs{\Pr\lbr{c(\xscript) = h^*(\xscript)}- \Pr\lbr{c(\xscript) = h_{\uscript}(\xscript)}} \leq 2\epsilon
        \end{equation*}
        Therefore, the proof is completed.
    \end{proof}


    \begin{remark}
        We emphasize that it is necessary for us to consider the guarantee of agnostic learning in a additive form rather than multiplicative form. Although Corollary~\ref{cor:optimal-agnostic-learning-implies-optimal-auditing} shows that the classification error, $\Pr\lbr{c(\xscript)\neq h(\xscript)}$, and the unfairness level, $\Pr\lbr{h}\abs{d_\D(c, h)}$, are dual to each other over $\HS_\mu^\D$, the affine relationship between them prohibits obtaining a guarantee on the unfairness from a multiplicative error. This also explains why the guarantee provided by \cite{diakonikolas2022learning} does not fit in our analysis.
    \end{remark}
    
    Now we are ready to prove Theorem \ref{thm:main-reduction}.
    
    \begin{proof}[Proof of Theorem \ref{thm:main-reduction}]
        To solve the auditing problem, we just need to solve the sequence of optimization problems, $\lbr{\Pg_{a,b}^\D(k, n)\cond k = 0, \ldots, n}$ as described in Proposition \ref{prop:approximating-auditing-with-simple-optimization}. We can solve each $\Pg_{a,b}^\D(k, n)$ with an additive error $\epsilon$ by calling the given oracle of learning halfspaces with the same strategy specified in Lemma~\ref{lma:reducing-auditing-to-agnostic-learning}. Eventually, we solve all of these optimization problems with an $2(b-a)/n + \epsilon$ additive error and a running time of $O(n)$ factor overhead compared with that of the oracle.
    \end{proof}
    
    \section{The Hardness of Auditing Under Gaussian Data}
    We show that the problem of auditing halfspaces subgroups under Gaussian distribution is computationally hard in two forms: the multiplicative form and additive form. The hardness results of both cases are obtained through reduction from auditing halfspace subgroups to the problem of continuous Learning With Error (cLWE) under Gaussian distribution, which is known to be as hard as LWE.
    \begin{proposition}[\cite{gupte2022continuous,diakonikolas2023near} Hardness of cLWE]\label{prop:hardness-of-continuous-LWE-under-gaussian-distribution}
        Given Assumption \ref{asp:sub-exponential-assumption-of-lwe}, for any $d\in\N$, any constants $\kappa\in\N, \alpha\in(0, 1), \beta\in\R_+$ and any $\log^\beta d\leq k\leq Cd$ where $C>0$ is a sufficiently small universal constant, the problem LWE$(d^{O(k^\alpha)}, \gaussian, \Sphere^{d - 1}, \gaussian_\sigma, \mod_T)$ over $\R^d$ with $\sigma\geq k^{-\kappa}$ and $T = 1/C'\sqrt{k\log d}$, where $C'>0$ is a sufficiently large universal constant, cannot be solved in time $d^{O(k^\alpha)}$ with $d^{-O(k^\alpha)}$ advantage
    \end{proposition} 
    
    We first show it is computationally hard to distinguish between halfspace subgroups that are evenly fair and halfspace subgroups among which there exists a slightly unfair subgroup.

    \begin{theorem}\label{thm:hardness-of-auditing-halfspace-subgroups}
        Under Assumption \ref{asp:sub-exponential-assumption-of-lwe}, for any $d\in\N$, any constants $\alpha\in(0, 1), \beta\in\R_+$, and any $\log^{\beta} d\leq k \leq cd$ where $c$ is a sufficiently small constant, there is no algorithm that runs in time $d^{O(k^\alpha)}$ and distinguishes between the following two cases of a joint distribution $\D$ of $(\xscript, c(\xscript))$ supported on $\R^d\times\lbr{-1, +1}$ with marginal $\D_{\xscript} = \gaussian(0, \identity)$, with $d^{-O(k^\alpha)}$ advantage:
        \begin{itemize}
            \item[(i)] \textbf{Alternative Hypothesis}: There exist non-negligibly unfair halfspace subgroups, specifically $\exists h\in\HS^d, \Pr_{\D}\lbr{h(\xscript) = 1}|d_\D(c, h)| = \Omega(1/\sqrt{k\log d})$.
            \item[(ii)] \textbf{Null Hypothesis}: All halfspace subgroups are perfectly fair, i.e., $\Pr_{\D}\lbr{h(\xscript) = 1}|d_\D(c, h)| = 0, \forall h\in \HS^d$.
        \end{itemize}
    \end{theorem}
    \begin{proof}
        We give an efficient method taking as input samples from a distribution $\D'$, that is either from the alternative hypothesis or the null hypothesis of LWE$(d^{O(k^\alpha)},\gaussian(0, \identity), \Sphere^{d-1}, \gaussian(0, \sigma),\mod_T)$ from Proposition \ref{prop:hardness-of-continuous-LWE-under-gaussian-distribution}, and generate samples from another distribution $\D$ with the following properties: if $\D'$ is from the alternative (resp. null) hypothesis of the LWE problem, then the resulting distribution $\D$ will satisfy the alternative (resp. null) hypothesis requirement of the theorem for the halfspace auditing problem.

        The reduction process can be formulated as follow: for a sample $(\xscript, y)$ from a instance $\D'$ of the problem LWE$(d^{O(k^\alpha)},\gaussian(0, \identity), \Sphere^{d-1}, \gaussian(0, \sigma),\mod_T)$ from Proposition \ref{prop:hardness-of-continuous-LWE-under-gaussian-distribution}, we simply output $(\xscript, c(\xscript))\sim\D$, where
        \begin{equation*}
            c(\xscript)=
            \begin{cases}
                +1,\quad \text{if }y\leq T/2\\
                -1,\quad \text{otherwise}
            \end{cases}
        \end{equation*}
        We argue that $\D$ satisfies the desired requirement stated above.

        For the alternative hypothesis case, let $\D'$ be from the alternative hypothesis case of the LWE. Let $\sscript$ be the secret vector in the LWE problem. We consider the following two halfspaces: 
        \begin{align*}
            h_1(\xscript) =& \sgn(\sscript(\transpose)\xscript - T/6)\\
            h_2(\xscript) =& \sgn(-\sscript(\transpose)\xscript + T/3)
        \end{align*}
        If we can show $\abs{\prob<\xscript\sim\D_{\xscript}>{h_1(\xscript) = 1}d_\D(c, h_1) + \prob<\xscript\sim\D_{\xscript}>{h_2(\xscript) = 1}d_\D(c, h_2)} = \Omega(T)$, then either $h = h_1$ or $h = h_2$ satisfies $\prob<\xscript\sim\D_{\xscript}>{h(\xscript) = 1}\abs{d_\D(c, h)} = \Omega(T)$, which implies the desired property of the alternative hypothesis we would like to prove. By Lemma \ref{lma:relationship-between-auditing-and-classification-error}, we have
        \begin{align*}
            &2\prob<\xscript\sim\D_{\xscript}>{h_1(\xscript) = 1}d_\D(c, h_1) + 2\prob<\xscript\sim\D_{\xscript}>{h_2(\xscript) = 1}d_\D(c, h_2)\\
            =& \underbrace{\prob{\neg c}(\prob{ \neg h_1} + \prob{ \neg h_2}) + \prob{ c}(\prob{ h_1} + \prob{ h_2})}_{I_1} \\
            &- (\underbrace{\prob{c(\xscript) = h_1(\xscript)} + \prob{c(\xscript) =h_2(\xscript)}}_{I_2})
        \end{align*}

        To bound $I_1, I_2$, we first examine the subset of domain where $h_1$ and $h_2$ agree, namely
        \begin{align*}
            B:=& \lbr{\xscript\in\R^d\cond h_1(\xscript) = h_2(\xscript)}\\
            =& \lbr{\xscript\in\R^d\cond h_1(\xscript)=1 \isect h_2(\xscript)=1}\\
            =&\lbr{\xscript\in\R^d\cond \sscript(\transpose)\xscript\in [T/6, T/3]}
        \end{align*}

        Then, for $I_1$, by the law of total probability, we have
        \begin{align*}
            I_1 =& \prob{c(\xscript) = -1}(\prob{ h_1(\xscript) = -1} + \prob{ h_2(\xscript) = -1} + \prob{\xscript\in B} - \prob{\xscript\in B})\\
            &+ \prob{ c(\xscript)=1}(\prob{ h_1(\xscript)=1} + \prob{ h_2(\xscript)=1\isect \xscript\notin B} + \prob{ h_2(\xscript)=1\isect \xscript\in B})\\
            \ceq[(i)]&\prob{c(\xscript) = -1}(1 - \prob{\xscript\in B}) + \prob{c(\xscript)=1}(1 + \prob{\xscript\in B})\\
            =& 1 + \prob{\xscript\in B}(\prob{c(\xscript)=1} - \prob{c(\xscript) = -1})\\
            =& 1 + \prob{\xscript\in B}(2\prob{c(\xscript)=1} - 1)
        \end{align*}
        where (i) is because $\lbr{\xscript\in\R^d\cond h_1(\xscript) = -1}, \lbr{\xscript\in\R^d\cond h_2(\xscript) = -1}, \lbr{\xscript\in B}$ are pairwise disjoint and their union equals to $\R^d$, $\lbr{\xscript\in\R^d\cond h_1(\xscript)=1}, \lbr{\xscript\in\R^d\cond h_2(\xscript)=1\isect \xscript\notin B}$ are disjoint and their union equals to $\R^d$; and since $\lbr{\xscript\in B}\subset \lbr{\xscript\in\R^d\cond h_2(\xscript) = 1}$ by definition, $\lbr{\xscript\in B}= \lbr{\xscript\in B\cond h_2(\xscript) = 1}$. 

        For $I_2$, because for any $\xscript\in B$, $h_1(\xscript) = h_2(\xscript) = 1$ by construction, and by the law of total probability, we have 
        \begin{align*}
            I_2=& \prob{c(\xscript) = h_1(\xscript)\isect \xscript\notin B} + \prob{c(\xscript) = h_2(\xscript)\isect \xscript\notin B} + 2\prob{c(\xscript) = 1\isect \xscript\in B}\\
            =& \prob{\xscript\notin B} + 2\prob{c(\xscript) = 1\isect \xscript\in B}\\
            =& 1 + \prob{c(\xscript) = 1\isect \xscript\in B} - \prob{c(\xscript) = -1\isect \xscript\in B}\\
            =& 1- \prob{\xscript\in B}(1 - 2\prob{c(\xscript) = 1\cond \xscript\in B}
        \end{align*}

        By the definition of $c$ as well as the Alternative case distribution of the LWE problem, $\lbr{\xscript\in \R^d\cond c(\xscript) = 1}$ is equivalent to $\lbr{\xscript\in\R^d\cond \mod_T(\sscript(\transpose)\xscript + z)\leq T/2}$ for some $z\sim\gaussian(0, \sigma^2)$. Further, we have
        \begin{equation*}
            \lbr{\xscript\in\R^d\cond \mod_T(\sscript(\transpose)\xscript + z)\leq T/2}\equiv \union_{k\in \Z}\lbr{\sscript(\transpose)\xscript + z\in(kT, kT + T/2]}
        \end{equation*}
        Notice that $\sscript(\transpose)\xscript + z$ is a one dimensional Gaussian random variable, which, by symmetry of Gaussian distribution, implies $\prob{c(\xscript) = 1} = \prob{\union_{k\in \Z}\lbr{\sscript(\transpose)\xscript + z\in(kT, kT + T/2]}} = 1/2$. Therefore, combining $I_1$ and $I_2$ gives
        \begin{align}
            I_1 - I_2 =& 2\prob{\xscript\in B}(\prob{c(\xscript)=1} - \prob{c(\xscript) = 1\cond\xscript\in B})\notag\\
            =&\Omega(T)(1/2- \prob{c(\xscript) = 1\cond\xscript\in B})\label{eq:difference-between-I1-and-I2}
        \end{align}
        where the last equation is because $\sscript(\transpose)\xscript\sim\gaussian(0, 1)$, hence, $\prob{\xscript\in B} = \prob{\sscript(\transpose)\xscript\in [T/6, T/3]} = \Omega(T)$. Since we were only concerned with showing $|I_1-I_2|$ is large, it suffices to show $\prob{c(\xscript) = 1\cond\xscript\in B} - 1/2 = \Omega(1)$. 
        
        For $\xscript\in B$, we have $\sscript(\transpose)\xscript\in[T/6, T/3]$, therefore $c(\xscript) = -1$ only if $\abs{z}\geq T/6$. Notice that $z\sim \gaussian(0, \sigma^2)$ and Proposition \ref{prop:hardness-of-continuous-LWE-under-gaussian-distribution} states that the LWE problem is hard for any fixed constant $\kappa\in \N$ and $\sigma\geq k^{-\kappa}$. Given the constant $\beta\in \R_+$ in this theorem, we can take $\kappa = \ceil{1/2\beta + 1/2 + 1}$, which is a fixed constant. Then, by Proposition \ref{prop:hardness-of-continuous-LWE-under-gaussian-distribution}, the LWE problem is hard for $\sigma = k^{-\kappa}\leq 1/(k^{3/2}\sqrt{\log d}) = o(T)$. Therefore, by a Gaussian tail bound, we have
        \begin{equation*}
            \prob<\xscript\sim\D_{\xscript}>{c(\xscript) = -1\cond \xscript\in B}\leq \prob<z\sim\gaussian(0, \sigma^2)>{\abs{z}\geq T/6} = o(1)
        \end{equation*}
        Plugging the above back into equation \eqref{eq:difference-between-I1-and-I2}, we can conclude that
        \begin{equation*}
            \prob<\xscript\sim\D_{\xscript}>{h_1(\xscript) = 1}d_\D(c, h_1) + \prob<\xscript\sim\D_{\xscript}>{h_2(\xscript) = 1}d_\D(c, h_2) = \Omega(T)
        \end{equation*}
        Thus, either $h = h_1$ or $h = h_2$ must satisfiy $\prob<\xscript\sim\D_{\xscript}>{h(\xscript) = 1}\abs{d_\D(c, h)} = \Omega(T)$, which completes the proof for the alternative hypothesis case.
        
        For the null hypothesis, we can immediately see that $\Pr_{\xscript\in \gaussian}\lbr{h}d_\gaussian(c,h) = 0, \forall h\in \HS^d$ because $c(\xscript)$ is independent from each $h\in\HS^d$. 

        It remains to verify the time lower bound and the distinguishing advantage for auditing halfspace subgroups. From Proposition \ref{prop:hardness-of-continuous-LWE-under-gaussian-distribution}, we know that under Assumption \ref{asp:sub-exponential-assumption-of-lwe}, for the problem LWE$(d^{O(k^\alpha)}, \gaussian(0, \identity), \Sphere^{d - 1}, \gaussian(0, \sigma^2), \mod_T)$ with any $\sigma\geq k^{-\kappa}$ (where $\kappa\in \N$ is a constant) and $T = 1/c'\sqrt{k\log d}$, where $c' > 0$ is a sufficiently large universal constant, the problem cannot be solved in $d^{O(k^\alpha)}$ time with $d^{-O(k^\alpha)}$ advantage. Therefore, under the same assumption, there is no algorithm that can solve the decision version of auditing problem w.r.t. halfspace subgroups in $d^{O(k^\alpha)}$ time with $d^{-O(k^\alpha)}$ advantage.
    \end{proof}
    
    Suppose an auditing algorithm is guaranteed to return us a $\gamma'$-unfair certificate (a halfspace) give a $\gamma$-unfair classifier $c$, where $\gamma'\leq\gamma\leq 1$, the following corollaries show that $\gamma'$ can never be close to $\gamma$.
    
    \begin{corollary}[multiplicative form]\label{cor:hardness-of-auditing-halfspace-under-gaussian-in-multiplicative-form}
        Given Assumption \ref{asp:sub-exponential-assumption-of-lwe}, there is no polynomial-time $1/\mathrm{poly}(d)$-approximation algorithm for constructive auditing for halfspace subgroups under Gaussian marginals in $\R^d$.
    \end{corollary}
    \begin{proof}
        Suppose there exists an auditing algorithm that guarantees to return a $\delta\gamma$-unfair certificate given a $\gamma$-unfair collection of halfspace subgroup and access to data with Gaussian marginal, where $\delta\in (0, 1)$. 
        
        For the alternative hypothesis case as described in Theorem \ref{thm:hardness-of-auditing-halfspace-subgroups}, given a $1/\sqrt{k\log d}$-unfair collection of halfspace subgroups, we run such an algorithm to obtain a $\delta/\sqrt{k\log d}$-unfair certificate, i.e., a halfspace $h$ such that $\prob<\xscript\sim\gaussian>{h(\xscript) = 1}|d_\gaussian(c, h)| \geq \delta/\sqrt{k\log d}$. By the Hoeffding Bound, we can verify that the empirical estimation of $\prob<\xscript\sim\gaussian>{h(\xscript) = 1}|d_\gaussian(c, h)|$ is $\varepsilon_1$-close to $\delta/\sqrt{k\log d}$ with high probability by drawing $O(1/\varepsilon_1^2)$ examples from the distribution constructed in the alternative hypothesis case.
        
        For the null hypothesis case, with the same argument, we can verify there is no $\varepsilon_2$-unfair subgroup with high probability given $O(1/\varepsilon_2^2)$ examples from the distribution in the null hypothesis case.

        However, we are given that no algorithm can distinguish the two cases in Theorem \ref{thm:hardness-of-auditing-halfspace-subgroups} with $d^{-O(k^\alpha)}$ advantage while running in time $d^{O(k^\alpha)}$. Thus, we cannot have 
        \begin{equation*}
            \frac{\delta}{\sqrt{k\log d}}-\varepsilon_1 - \varepsilon_2 = d^{-O(k^\alpha)}
        \end{equation*}
        for any $\varepsilon_1, \varepsilon_2 \leq d^{-O(k^\alpha)}$ and any $\log^{\beta} d\leq k \leq cd$. Therefore, we can infer that $\delta$ cannot be as large as $d^{-O(k^\alpha)}\sqrt{k\log d} = d^{-O(k^\alpha)}$, which implies the desired result, since $k\geq \log^{\beta} d$ and hence $\delta<1/d^{\log^{\alpha\beta}d}$ which is negligible.
    \end{proof}

    \begin{corollary}[additive form]\label{cor:hardness-of-auditing-halfspace-under-gaussian-in-additive-form}
        Given Assumption \ref{asp:sub-exponential-assumption-of-lwe}, for any constants $\alpha\in(0, 1), \beta\in\R_+$, and any $C/\sqrt{d\log d}\leq \epsilon\leq c'/\log^{(1 + \beta)/2}d$ where $C$ is a sufficiently large constant and $c'$ is a sufficiently small constant, no auditing algorithm can return a unfair certificate for halfspace subgroups in $\R^d$ with an additive error $\epsilon$ under Gaussian marginals and runs in time $d^{O(1/(\epsilon^2\log d)^{\alpha})}$.
    \end{corollary}
    \begin{proof}
       Suppose there exists an auditing algorithm that guarantees to return a $\gamma - \epsilon$-unfair certificate given a $\gamma$-unfair collection of halfspace subgroup and access to data with Gaussian marginal, where $\epsilon\in(0, 1)$.

       Similar to the proof of Corollary \ref{cor:hardness-of-auditing-halfspace-under-gaussian-in-multiplicative-form}, given a $1/\sqrt{k\log d}$-unfair collection of halfspace subgroups, we run such an algorithm to obtain a $(1/\sqrt{k\log d} - \epsilon)$-unfair certificate. Observe that, if $\epsilon = c'/\sqrt{k\log d}$ for some sufficiently small constant $c'$, we can solve the testing problem in Theorem \ref{thm:hardness-of-auditing-halfspace-subgroups} within time $d^{O(k^\alpha)}$ by running this algorithm as well as drawing enough examples to estimate the unfairness of the returned certificates from the two cases respectively. On the other hand, given $\epsilon = c'/\sqrt{k\log d}$, we can rewrite $d^{O(k^\alpha)} = d^{O(1/(\epsilon^2\log d)^{\alpha})}$. 
       
       However, Theorem \ref{thm:hardness-of-auditing-halfspace-subgroups} tells that the above case is impossible for any $C/\sqrt{d\log d}\leq \epsilon\leq c'/\log^{(1 + \beta)/2}d$, where $C$ is a sufficiently large constant.
    \end{proof}

    Besides the general general auditing problem, we also consider the ``non-constructive auditing'' problem as in Definition \ref{def:non-constructive-auditing}, where the algorithm is only required to tell if there exists an unfair subgroup without returning the unfair certificate. Actually, it turns out any non-constructive auditing algorithm can distinguish the two cases in Theorem \ref{thm:hardness-of-auditing-halfspace-subgroups}.

    \begin{corollary}[non-constructive auditing is hard]\label{cor:non-constructively-audit-halfspace-subgroups-under-gaussian-data-is-hard}
        Given Assumption \ref{asp:sub-exponential-assumption-of-lwe}, for any constants $\alpha\in(0, 1), \beta\in\R_+$, and any $C/\sqrt{d\log d}\leq \epsilon\leq c'/\log^{(1 + \beta)/2}d$ where $C$ is a sufficiently large constant and $c'$ is a sufficiently small constant, no auditing algorithm can tell if there exists a unfair certificate for halfspace subgroups in $\R^d$ with
        \begin{itemize}
            \item an additive error $\epsilon$ under Gaussian marginals and running in time $d^{O(1/(\epsilon^2\log d)^{\alpha})}$.
            \item or a multiplicative approximation factor of $1/\mathrm{poly}(d)$ and running in polynomial time.
        \end{itemize}
    \end{corollary}
    \begin{proof}
        Suppose there exists an auditing algorithm that can either tell if a $\delta\gamma$-unfair certificate or a $\gamma - \epsilon$-unfair certificate exists given a $\gamma$-unfair collection of halfspace subgroup and access to data with Gaussian marginal, where $\delta, \epsilon\in(0, 1)$. With the same argument as that of Corollary \ref{cor:hardness-of-auditing-halfspace-under-gaussian-in-multiplicative-form} and \ref{cor:hardness-of-auditing-halfspace-under-gaussian-in-additive-form}, we can achieve the desired results.
    \end{proof}
    
    To the best of our knowledge, there does not exist any PTAS for properly learning general halfspaces in agnostic model with guarantees of additive error close to $O(1/\sqrt{\log d})$. However, in the next section, we will show that if we restrict out attention to just homogeneous halfspaces under a standard normal distribution, it is possible to achieve additive error of $O(1/\log^{1/C} d)$ for some constant $C > 2$.

\section{AUDITING VIA AGNOSTIC LEARNING UNDER GAUSSIAN DISTRIBUTION}
    In this section, we present our algorithmic results. Our approach is based on Theorem \ref{thm:main-reduction}: auditing over subgroups determined by halfspaces can be accomplished by solving a sequence of simpler tasks of learning halfspaces. As a result, we are able to take advantage of existing agnostic learning methods to solve the auditing problem.
    
    Assuming there exists an efficient oracle for agnostic learning, algorithm \ref{alg:fairness-auditing} will eventually return a halfspace $h'$ as a certificate of the subgroup that has the highest unfairness level. We show the correctness, time and sample complexity of algorithm \ref{alg:fairness-auditing} in Theorem \ref{thm:auditing-framework}.
    \begin{algorithm}[ht]
        \caption{Fairness Auditing}\label{alg:fairness-auditing}
        \SetKwInOut{Input}{Input}
        \Input{$n, a, b, \epsilon,\delta,\D$, classifier $c$, oracle $\bigO$}
        \KwResult{$\mu', h'$}
        $\hat{\mathcal{X}} \gets $ draw $N(d, \epsilon, \delta)$ i.i.d. samples from $\D$\;
        $\hat{\D^+} \gets \lbr{\hat{\mathcal{X}}, c(\hat{\mathcal{X}})}$\;
        $\hat{\D^-} \gets \lbr{\hat{\mathcal{X}}, -c(\hat{\mathcal{X}})}$\;
        $\mu \gets a$\;
        $( \mu', h') \gets (1, c)$\;
        \While{$\mu \leq b$}{
            $h_\mu^+ \gets \bigO(\epsilon,\delta/2n, \mu, \hat{\D^+})$\;
            $h_\mu^- \gets \bigO(\epsilon,\delta/2n, \mu, \hat{\D^-})$\;
            \If{$\abs{d_\D(c, h_\mu^+)} < \abs{d_\D(c, h_\mu^-)}$}{
                $h_\mu^+ \gets h_\mu^-$\;
            }
            \If{$ \mu'\abs{d_\D(c, h')}\leq \mu\abs{d_\D(c, h_\mu^+)}$}{
                $(\mu', h') \gets (\mu, h_\mu^+)$\;
            }
            $\mu \gets \mu + (b - a)/n$\;
        }
    \end{algorithm}
    
    \begin{theorem}[Auditing Framework]\label{thm:auditing-framework}
        Given any binary classifier $c:\R^d\rightarrow\lbr{-1, +1}$, a data distribution $\D$ whose 1-dimensional marginals have continuous cumulative distribution functions, and collections of halfspaces $\lbr{\HS_\mu^\D\cond \mu > 0}$ over $\R^d$, if there exists an oracle $\bigO$ that takes $\epsilon, \delta\, \mu \in(0, 1)$ and $N(d, \epsilon, \delta)$ labeled i.i.d. samples from $\D$ in the form of $(\xscript, c(\xscript))$, runs in time $T(d, \epsilon, \delta)$, and returns a halfspace $h_\mu$ such that, with at least $1- \delta$ probability
        \begin{equation*}
            \Pr_{\xscript\sim\D}\lbr{h_\mu(\xscript)\neq c(\xscript)}\leq \min_{h\in \HS_\mu^\D}\Pr_{\xscript\sim\D}\lbr{h(\xscript)\neq c(\xscript)} + \epsilon
        \end{equation*}
        then there exists an algorithm that takes $n\in\Z^+$, $0 < a\leq b<1$, $\epsilon, \delta\in(0,1)$ and $O(N(d, \epsilon, \delta/n))$ labeled i.i.d. samples from $\D$, runs in time $O(nT(d, \epsilon, \delta/n))$, returns a halfspace $h'$ as a certificate such that $a\leq \Pr_{\xscript\sim\D}\lbr{h'}\leq b$ as well as
        \begin{equation*}
            \Pr_{\xscript\sim\D}\lbr{h'}\abs{d_\D(c,h')}\geq \max_{h\in\HS^d}\Pr_{\xscript\sim\D}\lbr{h}\abs{d_\D(c, h)}-O(\epsilon)
        \end{equation*}
        with at least $1 - \delta$ probability.
    \end{theorem}
    \begin{proof}
        Let's notice that, although each iteration of the loop in algorithm \ref{alg:fairness-auditing} solves $\min_{h\in\HS_\mu^\D}\Pr\lbr{c(\xscript)\neq h(\xscript)}$ and $\max_{h\in\HS_\mu^\D}\Pr\lbr{c(\xscript)\neq h(\xscript)}$, it is essentially equivalent to solving $\max_{h\in\HS_\mu^\D}\abs{d_\D(c, h)}$ according to lemma \ref{lma:reducing-auditing-to-agnostic-learning}. As the oracle returns a halfspace with additive error smaller than $\epsilon$ with probability at least $1 - \delta$, we have that
        \begin{equation*}
            \max(\abs{d_\D(c, h_\mu^+)}, \abs{d_\D(c, h_\mu^-)})\geq \max_{h\in \HS_\mu^\D}\abs{d_\D(c, h_\mu^+)} - \frac{\epsilon}{\mu}
        \end{equation*}
        with probability at least $1 - \delta/n$ because of Lemma~\ref{lma:reducing-auditing-to-agnostic-learning} as well as union bound.
    
        Across all iterations, the algorithm maximizes $\mu\abs{d_\D(c, h_\mu^+)}$ over $\HS_\mu^\D$ for $\mu$ increase from $a$ to $b$ with step size $(b -a)/n$. With a union bound over all $n$ iterations, we can have the same additive error $\epsilon$ in every iterations, with probability at least $1 - \delta$. As a result, the algorithm equivalently solves
        \begin{align*}
            \max_{h\in\HS^d}\quad &\Pr_{\xscript\in\D}\lbr{h(\xscript) = 1}\abs{d_\D(c, h)}\\
            s.t.\quad &a\leq\Pr_{\xscript\in\D}\lbr{h(\xscript) = 1}\leq b
        \end{align*}
        with probability at least $1 - \delta$ for an additive error at most $2(b-a)/n + \epsilon$ according to Proposition~\ref{prop:approximating-auditing-with-simple-optimization}, which completes the proof.
    \end{proof}
    While our framework heavily relies on the methods of agnostic learning with small additive error, unfortunately, there are no known methods for learning general halfspaces that can achieve additive error better than a constant, even under distributions as nice as standard normal ones.
    
    However, if we restrict our audit to the class of homogeneous halfspaces, Diakonikolas et al.~\cite{diakonikolas2021agnostic} proposed an agnostic learning PTAS for homogeneous halfspaces under Gaussian data. That is, we only audit for subgroups with probability mass $1/2$. 
    \begin{lemma}[Learning Homogeneous Halfspaces \cite{diakonikolas2021agnostic}]\label{lma:learning-homogeneous-halfspaces-in-agnostic-setting}
        Let $\D$ be a distribution on labeled examples $(\xscript, y)\in\R^d\times \lbr{-1, + 1}$ whose $\xscript$-marginal is $\gaussian(0, \identity)$. There exists an algorithm that, given $\tau, \epsilon, \delta > 0$, and $N=d^{\mathrm{poly}(1/\tau)}\mathrm{poly}(1/\epsilon)\log(1/\delta)$ i.i.d. samples from $\D$, the algorithm runs in time $\mathrm{poly}(N, d)$, and computes a halfspace $h_{\vscript<>}$ such that, with probability at least $1 - \delta$, it holds $\Pr_\D\lbr{y\neq h_{\vscript<>}(\xscript)}\leq (1 + \tau)\min_{h\in \HS_{1/2}^\gaussian}\Pr_\D\lbr{y\neq h(\xscript)} + \epsilon$.
    \end{lemma}
    
    We show our algorithmic guarantee of a PTAS in the following corollary.
    
    \begin{corollary}[Auditing Under Gaussian]
        Given any binary classifier $c:\R^d\rightarrow\lbr{-1, +1}$, a data distribution $\gaussian(0, \identity)$ and a collection of halfspaces $\HS_{1/2}^\gaussian$ over $\R^d$, there exists an auditing algorithm that takes $\epsilon, \delta >0$ and $N=d^{\mathrm{poly(1/\epsilon)}}\mathrm{poly}(1/\epsilon)\log(1/\delta)$ labeled i.i.d. example from $\gaussian(0, \identity)$ in the form of $(\xscript, c(\xscript))$, runs in time $\mathrm{poly}(N, d)$, and returns a halfspace $h'$ as a certificate such that $\Pr_{\xscript\sim\D}\lbr{h'} = 1/2$ and
        \begin{equation*}
            \abs{d_\gaussian(c,h')}\geq \max_{h\in\HS_{1/2}^\gaussian}\abs{d_\gaussian(c, h)}-2\epsilon
        \end{equation*}
        with at least $1 - \delta$ probability.
    \end{corollary}
    \begin{proof}
        Simply running Algorithm \ref{alg:fairness-auditing} with the same set of parameters except that $\D = \gaussian(0, \identity)$, $n = 1$, $a=b=1/2$ and the oracle being as described by Lemma \ref{lma:learning-homogeneous-halfspaces-in-agnostic-setting} for $\tau = \epsilon$ will give us the desired results.
    \end{proof}

\section{FUTURE WORK}
    The major drawback of our result is still the lack of approaches of learning halfspaces with sub-constant error guarantee for more general distributions. Therefore, a major direction for fairness auditing remains to develop an agnostic learning method with additive error guarantees for broader classes, such as log-concave distributions---subject to the constraints of Corollary \ref{cor:hardness-of-auditing-halfspace-under-gaussian-in-additive-form}/Diakonikolas et al.~\cite{diakonikolas2023near}. Even a computationally efficient learning algorithm for general halfspaces that can achieve additive error close to $O(1/\sqrt{\log d})$ under Gaussian distributions would be an interesting improvement.

    An alternative direction is to seek stronger guarantees for conjunctions on such families of distributions. Conjunctions are more natural in the context of auditing, and their relative lack of expressive power might enable a better guarantee.

    \bibliographystyle{unsrt}  
    \bibliography{refs} 
\end{document}

%% file: commands.tex
\usepackage{xparse}
\newtheorem{theorem}{Theorem}[section]
\newtheorem{proposition}[theorem]{Proposition}
\newtheorem{lemma}[theorem]{Lemma}
\newtheorem{corollary}[theorem]{Corollary}
\newtheorem{fact}[theorem]{Fact}
\newtheorem{definition}[theorem]{Definition}
\newtheorem{remark}[theorem]{Remark}
\newtheorem{example}[theorem]{Example}
\newtheorem{assumption}[theorem]{Assumption}

\newcommand{\union}{\ \cup\ }
\newcommand{\isect}{\ \cap\ }
\newcommand{\cond}{\ | \ }
\newcommand{\ceil}[1]{\ensuremath \left\lceil #1 \right\rceil}

\NewDocumentCommand\prob{sD<>{}m}{
    \IfBooleanTF{#1}{
        \IfNoValueTF{#2}{
            \tilde{\Pr}\lbr{#3}
        }{
            \tilde{\Pr}_{#2}\lbr{#3}
        }
    }{
        \IfNoValueTF{#2}{
            \Pr\lbr{#3}
        }{
            \Pr_{#2}\lbr{#3}
        }
    }
}

\newcommand{\C}{\mathcal{C}}
\newcommand{\D}{\mathcal{D}}

\newcommand{\HS}{\mathcal{H}}

\newcommand{\R}{\mathbb{R}}
\newcommand{\N}{\mathbb{N}}
\newcommand{\Z}{\mathbb{Z}}

\newcommand{\Sphere}{\mathbb{S}}
\newcommand{\Pg}{\mathcal{P}}

\NewDocumentCommand{\GI}{}{\mathcal{G}}
\newcommand{\transpose}{\top}

\DeclareMathOperator*{\sgn}{\mathrm{sgn}}
\newcommand{\gaussian}{\mathcal{N}}

\newcommand{\identity}{\mathrm{I}}
\NewDocumentCommand{\algo}{}{\mathcal{A}}
\newcommand{\opt}{\mathrm{opt}}
\newcommand{\dom}{\mathbf{dom}}

\newcommand{\abs}[1]{\left|#1\right|}
\newcommand{\sbr}[1]{\left(#1\right)}

\newcommand{\lbr}[1]{\{ #1 \}}

\newcommand{\cleq}[1][]{\overset{\mathrm{#1}}{\leq}}

\newcommand{\ceq}[1][]{\overset{\mathrm{#1}}{=}}

\NewDocumentCommand{\f}{o}{\IfValueTF{#1}{f_i(#1)}{f_i}}
\NewDocumentCommand{\fa}{o}{\IfValueTF{#1}{\bar{f}\sbr{#1}}{\bar{f}}}

\NewDocumentCommand{\g}{o}{\IfValueTF{#1}{g_i(#1)}{g_i}}
\NewDocumentCommand{\ga}{o}{\IfValueTF{#1}{\bar{g}\sbr{#1}}{\bar{g}}}

\newcommand{\norm}[1]{\lVert {#1} \rVert}

\NewDocumentCommand\vvector{}{\mathbf{v}}
\NewDocumentCommand\vscript{sD<>{1:d}oD(){}}{
	\IfBooleanTF#1
		{\IfNoValueTF{#3}
			{\bar{\vvector}_{#2}^{#4}}
			{\bar{\vvector}_{#2}^{(#3)#4}}
		}
		{\IfNoValueTF{#3}
			{\vvector_{#2}^{#4}}
			{\vvector_{#2}^{(#3)#4}}
		}
}
\NewDocumentCommand\vasync{sO{i}}{
	\IfBooleanTF{#1}
		{(\vscript<0>[#2-1],\vscript*[#2])}
		{(\vscript<0>[#2-1],\vscript[#2])}
}
\NewDocumentCommand\evector{}{\mathbf{e}}
\NewDocumentCommand\escript{sD<>{}oD(){}}{
	\IfBooleanTF#1
		{\IfNoValueTF{#3}
			{\bar{\evector}_{#2}^{#4}}
			{\bar{\evector}_{#2}^{(#3)#4}}
		}
		{\IfNoValueTF{#3}
			{\evector_{#2}^{#4}}
			{\evector_{#2}^{(#3)#4}}
		}
}
\NewDocumentCommand\elist{sD<>{}O{1}mO{}}{
	\IfBooleanTF#1
		{\escript*<#2>[#3](#5), \ldots, \escript*<#2>[#4](#5)}
		{\escript<#2>[#3](#5), \ldots, \escript<#2>[#4](#5)}
}
\NewDocumentCommand\wvector{}{\mathbf{w}}
\NewDocumentCommand\wscript{sD<>{}oD(){}}{
	\IfBooleanTF#1
		{\IfNoValueTF{#3}
			{\bar{\wvector}_{#2}^{#4}}
			{\bar{\wvector}_{#2}^{(#3)#4}}
		}
		{\IfNoValueTF{#3}
			{\wvector_{#2}^{#4}}
			{\wvector_{#2}^{(#3)#4}}
		}
}
\NewDocumentCommand\wlist{sD<>{}O{1}mO{}}{
	\IfBooleanTF#1
		{\wscript*<#2>[#3](#5), \ldots, \wscript*<#2>[#4](#5)}
		{\wscript<#2>[#3](#5), \ldots, \wscript<#2>[#4](#5)}
}
\NewDocumentCommand\uvector{}{\mathbf{u}}
\NewDocumentCommand\uscript{sD<>{}oD(){}}{
	\IfBooleanTF#1
		{\IfNoValueTF{#3}
			{\bar{\uvector}_{#2}^{#4}}
			{\bar{\uvector}_{#2}^{(#3)#4}}
		}
		{\IfNoValueTF{#3}
			{\uvector_{#2}^{#4}}
			{\uvector_{#2}^{(#3)#4}}
		}
}
\NewDocumentCommand\ulist{sD<>{}O{1}mO{}}{
	\IfBooleanTF#1
		{\uscript*<#2>[#3](#5), \ldots, \uscript*<#2>[#4](#5)}
		{\uscript<#2>[#3](#5), \ldots, \uscript<#2>[#4](#5)}
}
\NewDocumentCommand\xvector{s}{
	\IfBooleanTF#1
		{\bar{\mathbf{x}}}
		{\mathbf{x}}
}
\NewDocumentCommand\xscript{sD<>{}oD(){}}{
	\IfBooleanTF#1
		{\IfNoValueTF{#3}
			{\xvector*_{#2}^{#4}}
			{\xvector*_{#2}^{(#3)#4}}
		}
		{\IfNoValueTF{#3}
			{\xvector_{#2}^{#4}}
			{\xvector_{#2}^{(#3)#4}}
		}
}
\NewDocumentCommand\svector{s}{
	\IfBooleanTF#1
		{\bar{\mathbf{s}}}
		{\mathbf{s}}
}
\NewDocumentCommand\sscript{sD<>{}oD(){}}{
	\IfBooleanTF#1
		{\IfNoValueTF{#3}
			{\svector*_{#2}^{#4}}
			{\svector*_{#2}^{(#3)#4}}
		}
		{\IfNoValueTF{#3}
			{\svector_{#2}^{#4}}
			{\svector_{#2}^{(#3)#4}}
		}
}
\NewDocumentCommand\slist{sD<>{}O{1}mO{}}{
	\IfBooleanTF#1
		{\sscript*<#2>[#3](#5), \ldots, \sscript*<#2>[#4](#5)}
		{\sscript<#2>[#3](#5), \ldots, \sscript<#2>[#4](#5)}
}
\NewDocumentCommand\yvector{s}{
	\IfBooleanTF#1
		{Y}
		{\mathbf{y}}
}
\NewDocumentCommand\yscript{sD<>{}oD(){}}{
	\IfBooleanTF#1
		{\IfNoValueTF{#3}
			{\yvector*_{#2}^{#4}}
			{\yvector*_{#2}^{(#3)#4}}
		}
		{\IfNoValueTF{#3}
			{\yvector_{#2}^{#4}}
			{\yvector_{#2}^{(#3)#4}}
		}
}
\NewDocumentCommand\ylist{sD<>{}O{1}mO{}}{
	\IfBooleanTF#1
		{\yscript*<#2>[#3](#5), \ldots, \yscript*<#2>[#4](#5)}
		{\yscript<#2>[#3](#5), \ldots, \yscript<#2>[#4](#5)}
  }
\NewDocumentCommand\zvector{s}{
	\IfBooleanTF#1
		{Z}
		{\mathbf{z}}
}
\NewDocumentCommand\zscript{sD<>{}oD(){}}{
	\IfBooleanTF#1
		{\IfNoValueTF{#3}
			{\zvector*_{#2}^{#4}}
			{\zvector*_{#2}^{(#3)#4}}
		}
		{\IfNoValueTF{#3}
			{\zvector_{#2}^{#4}}
			{\zvector_{#2}^{(#3)#4}}
		}
}
\NewDocumentCommand\zlist{sD<>{}O{1}mO{}}{
	\IfBooleanTF#1
		{\zscript*<#2>[#3](#5), \ldots, \zscript*<#2>[#4](#5)}
		{\zscript<#2>[#3](#5), \ldots, \zscript<#2>[#4](#5)}
}
\NewDocumentCommand\muscript{sD<>{}oD(){}}{
	\IfBooleanTF#1
		{\IfNoValueTF{#3}
			{\Mu_{#2}^{#4}}
			{\Mu_{#2}^{(#3)#4}}
		}
		{\IfNoValueTF{#3}
			{\mu_{#2}^{#4}}
			{\mu_{#2}^{(#3)#4}}
		}
}
\NewDocumentCommand\siscript{sD<>{}oD(){}}{
	\IfBooleanTF#1
		{\IfNoValueTF{#3}
			{\Sigma_{#2}^{#4}}
			{\Sigma_{#2}^{(#3)#4}}
		}
		{\IfNoValueTF{#3}
			{\sigma_{#2}^{#4}}
			{\sigma_{#2}^{(#3)#4}}
		}
}

\NewDocumentCommand\psinorm{mD<>{2}}{
    \norm{#1}_{\psi_{#2}}
}

\DeclareMathOperator*{\argmin}{argmin}
\DeclareMathOperator*{\argmax}{argmax}

\NewDocumentCommand{\bigO}{}{\mathcal{O}}

